\documentclass[11pt]{article}

\usepackage{fullpage,times,url}

\usepackage{amsthm,amsfonts,amsmath,amssymb,epsfig,color,float,graphicx,verbatim}
\usepackage{algorithm,algorithmic}

\newtheorem{theorem}{Theorem}

\newtheorem{lemma}{Lemma}
\newtheorem{corollary}{Corollary}

\newcommand{\reals}{\mathbb{R}}
\newcommand{\E}{\mathbb{E}}

\newcommand{\be}{\mathbf{e}}
\newcommand{\bx}{\mathbf{x}}
\newcommand{\bw}{\mathbf{w}}

\newcommand{\bv}{\mathbf{v}}

\newcommand{\Ocal}{\mathcal{O}}

\newcommand{\norm}[1]{\|#1\|}
\newcommand{\inner}[1]{\langle#1\rangle}

\newcommand{\subsecref}[1]{Subsection~\ref{#1}}

\renewcommand{\eqref}[1]{Eq.~(\ref{#1})}
\newcommand{\lemref}[1]{Lemma~\ref{#1}}
\newcommand{\corref}[1]{Corollary~\ref{#1}}
\newcommand{\thmref}[1]{Thm.~\ref{#1}}

\title{Convergence of Stochastic Gradient Descent for PCA}
\author{Ohad Shamir\\Weizmann Institute of Science\\\texttt{ohad.shamir@weizmann.ac.il}}

\date{}

\begin{document}

\maketitle

\begin{abstract}
We consider the problem of principal component analysis (PCA) in a
streaming stochastic setting, where our goal is to find a direction of
approximate maximal variance, based on a stream of i.i.d. data points in
$\reals^d$. A simple and computationally cheap algorithm for this is
stochastic gradient descent (SGD), which incrementally updates its estimate
based on each new data point. However, due to the non-convex nature of the
problem, analyzing its performance has been a challenge. In particular,
existing guarantees rely on a non-trivial eigengap assumption on the
covariance matrix, which is intuitively unnecessary. In this paper, we
provide (to the best of our knowledge) the first eigengap-free convergence
guarantees for SGD in the context of PCA. This also partially resolves an
open problem posed in \cite{hardt2014noisy}. Moreover, under an eigengap assumption, we show that the same techniques lead to new SGD convergence guarantees with better dependence on the eigengap.
\end{abstract}

\section{Introduction}

Principal component analysis (PCA)
\cite{pearson1901liii,hotelling1933analysis} is a fundamental tool in data
analysis and visualization, designed to find the subspace of largest variance
in a given dataset (a set of points in Euclidean space). We focus on a simple
stochastic setting, where the data $\bx_1,\bx_2,\ldots\in \reals^d$ is
assumed to be drawn i.i.d. from an unknown underlying distribution, and our
goal is to find a direction of approximately maximal variance. This can be
written as the optimization problem
\begin{equation}\label{eq:objPCA}
\min_{\bw:\norm{\bw}=1} -\bw^\top \E[\bx\bx^\top] \bw,
\end{equation}
or equivalently, finding an approximate leading eigenvector of the covariance
matrix $\E[\bx\bx^\top]$.

The conceptually simplest method for this task, given $m$ sampled points
$\bx_1,\ldots,\bx_m$, is to construct the empirical covariance matrix
$\frac{1}{m}\sum_{i=1}^{m}\bx_i\bx_i^\top$, and compute its leading
eigenvector by an eigendecomposition. Based on concentration of measure
arguments, it is not difficult to show that this would result in an
$\Ocal(\sqrt{1/m})$-optimal solution to \eqref{eq:objPCA}. Unfortunately, the
runtime of this method is $\Ocal(md^2+d^3)$. In large-scale applications,
both $m$ and $d$ might be huge, and even forming the $d\times d$ covariance
matrix, let alone performing an eigendecomposition, can be computationally
prohibitive. A standard alternative to exact eigendecomposition is iterative
methods, such as power iterations or the Lanczos method, which require
performing multiple products of a vector with the empirical covariance
matrix. Although this doesn't require computing and storing the matrix
explicitly, it still requires multiple passes over the data, whose number may
scale with eigengap parameters of the matrix or the target accuracy
\cite{kuczynski1992estimating,musco2015stronger}. Recently, new randomized
algorithms for this problem were able to significantly reduce the required
number of passes, while maintaining the ability to compute high-accuracy
solutions \cite{shamir2015stochastic,shamir2015fast,GarHaz15,JiKaMuNeSi15}.

In this work, we consider the efficacy of algorithms which perform a
\emph{single pass} over the data, and in particular, stochastic gradient
descent (SGD). For solving \eqref{eq:objPCA}, SGD corresponds to initializing
at some unit vector $\bw_0$, and then at each iteration $t$ perform a
stochastic gradient step with respect to $\bx_t\bx_t^\top$ (which is an
unbiased estimate of $\E[\bx\bx^\top]$), followed by a projection to the unit
sphere:
\[
\bw_{t} := (I+\eta\bx_t\bx_t^\top)\bw_{t-1}~~,~~\bw_t:=\bw_t/\norm{\bw_t}.
\]
Here, $\eta$ is a step size parameter. In the context of PCA, this is also
known as Oja's method \cite{oja1982simplified,oja1985stochastic}. The
algorithm is highly efficient in terms of memory and runtime per iteration,
requiring storage of a single $d$-dimensional vector, and performing only
vector-vector and a vector-scalar products in each iteration.

In the world of convex stochastic optimization and learning, SGD has another
remarkable property: Despite it being a simple, one-pass algorithm, it is
essentially (worst-case) statistically optimal, attaining the same
statistical estimation error rate as exact empirical risk minimization
\cite{bousquet2008tradeoffs,shalev2009stochastic,shalev2014understanding}.
Thus, it is quite natural to ask whether SGD also performs well for the PCA
problem in \eqref{eq:objPCA}, compared to statistically optimal but
computationally heavier methods.

The study of SGD (or variants thereof) for PCA has gained interest in recent
years, with some notable examples including
\cite{ACLS12,balsubramani2013fast,arora2013stochastic,mitliagkas2013memory,hardt2014noisy,de2015global,JiKaMuNeSi15}.
While experimentally SGD appears to perform reasonably well, its theoretical
analysis has proven difficult, due to the non-convex nature of the objective
function in \eqref{eq:objPCA}. Remarkably, despite this non-convexity,
finite-time convergence guarantees have been obtained under an eigengap
assumption -- namely, that the difference between the largest and 2nd-largest
eigenvalues of $\E[\bx\bx^\top]$ are separated by some fixed value
$\lambda>0$. For example, \cite{de2015global} require
$\Ocal(d/\lambda^2\epsilon)$ iterations to ensure with high probability that
one of the iterates is $\epsilon$-optimal. \cite{JiKaMuNeSi15} require $\Ocal(1/\lambda^2+1/\lambda\epsilon)$ iterations, provided we begin close enough to an optimal solution. 

Nevertheless, one may ask whether the eigengap assumption is indeed
necessary, if our goal is simply to find an approximately optimal solution of  \eqref{eq:objPCA}. Intuitively, if $\E[\bx\bx^\top]$ has two equal (or near equal) top eigenvalues, then we may still expect to get a solution which lies close to the subspace of these two top eigenvalues, and approximately minimizes \eqref{eq:objPCA}, with the runtime not dependent on any eigengap. Unfortunately, existing results tell us nothing about this regime, and not just for minor technical reasons: These results are based on tracking the geometric convergence of the SGD iterates $\bw_t$ to a leading
eigenvector of the covariance matrix. When there is no eigengap, there is also no single eigenvector to converge to, and such a geometric approach does not seem to work. Getting an eigengap-free analysis has also been posed as an open problem in \cite{hardt2014noisy}. We note that while there are quite a
few other single-pass, eigengap-free methods for this problem, such as
\cite{warmuth2006online,warmuth2008randomized,nie2013online,boutsidis2015online,GaHaMa15,kotlowski2015pca},
their memory and runtime-per iteration requirements are much higher than SGD, often $\Ocal(d^2)$ or worse.

In this work, we study the convergence of SGD for PCA, using a different technique that those employed in previous works, with the following main results:
\begin{itemize}
	\item We provide the first (to the best of our knowledge) SGD convergence guarantee  which does not pose an eigengap assumption. Roughly speaking, we prove
	that if the step size is chosen appropriately, then after $T$ iterations
	starting from random initialization, with positive probability, SGD returns an
	$\tilde{\Ocal}(\sqrt{p/T})$-optimal\footnote{Throughout, we use $\Ocal$,
	$\Omega$ to hide constants, and $\tilde{\Ocal}$, $\tilde{\Omega}$ to hide
	constants and logarithmic factors.} solution of \eqref{eq:objPCA}, where $p$ is a parameter depending on how the algorithm is initialized: 
	\begin{itemize}
	    \item If the algorithm is initialized from a warm-start point $\bw_0$ such that $\frac{1}{\inner{\bv,\bw_0}^2}\leq \Ocal(1)$ for some leading eigenvector $\bv$ of the covariance matrix, then $p=\Ocal(1)$.
		\item Under uniform random initialization on the unit Euclidean sphere, $p=\Ocal(d)$, where $d$ is the dimension.
		\item Using a more sophisticated initialization (requiring the usage of the first $\Ocal(d)$ iterations, but no warm-start point), $p=\tilde{\Ocal}(n_A)$, where $n_A$ is the \emph{numerical rank} of the covariance matrix. The numerical rank is a relaxation of the standard notion of rank, is always at most $d$ and can be considered a constant under some mild assumptions.  
   \end{itemize}
   	\item In the scenario of a positive eigengap $\lambda>0$, and using a similar proof technique, we prove an SGD convergence guarantee of $\Ocal(p/\lambda T)$ (where $p$ is as above) with positive probability. This guarantee is optimal in terms of dependence on $T,\lambda$, and in particular, has better dependence on $\lambda$ compared to all previous works on SGD-like methods we are aware of ($1/\lambda$ as opposed to $1/\lambda^2$). 
\end{itemize}
Unfortunately, a drawback of our guarantees is that they only hold with 	rather low probability: $\Omega(1/p)$, which can be small if $p$ is large. Formally, this can be overcome by repeating the algorithm $\tilde{\Ocal}(p)$ times, which ensures that with high probability, at least one of the outputs will be close to optimal. However, we suspect that these low probabilities are an artifact of our proof technique, and resolving it is left to future work.

\section{Setting}

We use bold-faced letters to denote vectors, and capital letters to denote
matrices. Given a matrix $M$, we let $\norm{M}$ denote its spectral norm, and
$\norm{M}_F$ its Frobenius norm.

We now present the formal problem setting, in a somewhat more general way
than the PCA problem considered earlier. Specifically, we study the problem
of solving
\begin{equation}\label{eq:obj}
\min_{\bw\in\reals^d:\norm{\bw}=1} -\bw^\top A \bw,
\end{equation}
where $d>1$ and $A$ is a positive semidefinite matrix, given access to a
stream of i.i.d. positive semidefinite matrices $\tilde{A}_t$ where
$\E[\tilde{A}_t]=A$ (e.g. $\bx_t\bx_t^\top$ in the PCA case). Notice that the
gradient of \eqref{eq:obj} at a point $\bw$ equals $2A\bw$, with an unbiased
stochastic estimate being $2\tilde{A}_t\bw$. Therefore, applying SGD to
\eqref{eq:obj} reduces to the following: Initialize at some unit-norm vector
$\bw_0$, and for $t=1,\ldots,T$, perform $\bw_{t} = (I+\eta
\tilde{A}_t)\bw_{t-1},\bw_t=\bw_t/\norm{\bw_t}$, returning $\bw_T$. In fact,
for the purpose of the analysis, it is sufficient to consider a formally
equivalent algorithm, which only performs the projection to the unit sphere
at the end:
\begin{itemize}
  \item Initialize by picking a unit norm vector $\bw_0$
  \item For $t=1,\ldots,T$, perform $\bw_{t} = (I+\eta
      \tilde{A}_t)\bw_{t-1}$
  \item Return $\frac{\bw_T}{\norm{\bw_T}}$
\end{itemize}
It is easy to verify that the output of this algorithm is mathematically
equivalent to the original SGD algorithm, since the stochastic gradient step
amounts to multiplying $\bw_{t-1}$ by a matrix independent of $\bw_{t-1}$,
and the projection just amounts to re-scaling. In both cases, we can write
the algorithm's output in closed form as
\[
\frac{\left(\prod_{t=T}^{1}(I+\eta \tilde{A}_t)\right)\bw_0}{\left\|\left(\prod_{t=T}^{1}(I+\eta \tilde{A}_t)\right)\bw_0\right\|}.
\]

\section{Convergence Without an Eigengap Assumption}

Our main result is the following theorem, which analyzes the performance of
SGD for solving \eqref{eq:obj}.

\begin{theorem}\label{thm:main}
  Suppose that
  \begin{itemize}
    \item For some leading eigenvector $\bv$ of $A$,
        $\frac{1}{\inner{\bv,\bw_0}^2} \leq p$ for some $p$ (assumed to be
        $\geq 8$ for simplicity).
    \item For some $b\geq 1$, both $\frac{\norm{\tilde{A}_t}}{\norm{A}}$
        and $\frac{\norm{\tilde{A}_t-A}}{\norm{A}}$ are at most $b$ with
        probability $1$.
  \end{itemize}
  If we run the algorithm above for $T$ iterations with
  $\eta=\frac{1}{b\sqrt{pT}}$ (assumed to be $\leq 1$), then with probability at least
  $\frac{1}{c p}$, the returned $\bw$ satisfies
  \[
  1-\frac{\bw^\top A \bw}{\norm{A}} ~\leq~ c'\frac{\log(T)b\sqrt{p}}{\sqrt{T}},
  \]
  where $c,c'$ are positive numerical constants.
\end{theorem}
The proof and an outline of its main ideas appears in \subsecref{subsec:proofmain}
below. Note that this is a \emph{multiplicative} guarantee on the
suboptimality of \eqref{eq:obj}, since we normalize by $\norm{A}$, which is
the largest magnitude \eqref{eq:obj} can attain. By multiplying both sides by
$\norm{A}$, we can convert this to an additive bound of the form
\[
\norm{A}-\bw^\top A \bw ~\leq~ c'\frac{\log(T)b'\sqrt{p}}{\sqrt{T}},
\]
where $b'$ is a bound on
$\max\left\{\norm{\tilde{A}_t},\norm{\tilde{A}_t-A}\right\}$. Also, note that
the choice of $\eta$ in the theorem is not crucial, and similar bounds (with
different $c,c'$) can be shown for other $\eta=\Theta(1/b\sqrt{pT})$.

The value of $p$ in the theorem depends on how the initial point $\bw_0$ is
chosen. One possibility, of course, is if we can initialize the algorithm from a ``warm-start'' point $\bw_0$ such that $\frac{1}{\inner{\bv,\bw_0}^2}\leq\Ocal(1)$, in which case the bound in the theorem becomes $\Ocal(\log(T)/\sqrt{T})$ with probability $\Omega(1)$. Such a $\bw_0$ may be given by some other algorithm, or alternatively, if we are interested in analyzing SGD in the regime where it is close to one of the leading eigenvectors. 

Of course, such an assumption is not always relevant, so let us turn to consider the performance without such a ``warm-start''. For example, the simplest and most common way to initialize $\bw_0$
is by picking it uniformly at random from the unit sphere. In that case, for any $\bv$, $\inner{\bv,\bw_0}^2=\Theta(1/d)$ with high constant probability\footnote{One way to see this is by  assuming w.l.o.g. that $\bv=\be_1$ and noting that the distribution of $\bw_0$ is the same as $\bw/\norm{\bw}$ where $\bw$ has a standard Gaussian distribution, hence $\inner{\bv,\bw_0}^2=w_1^2/\sum_j w_j^2$, and by using standard concentration tools it can be shown that the numerator is $\Theta(1)$ and the denominator is $\Theta(d)$ with high probability.}, so the theorem above
applies with $p=\Ocal(d)$:

\begin{corollary}\label{cor:vanilla}
If $\bw_0$ is chosen uniformly at random from the unit sphere in $\reals^d$,
then \thmref{thm:main} applies with $p=\Ocal(d)$, and the returned
$\bw$ satisfies, with probability at least $\Omega(1/d)$,
  \[
  1-\frac{\bw^\top A \bw}{\norm{A}} ~\leq~ \Ocal\left(\frac{\log(T)b\sqrt{d}}{\sqrt{T}}\right),
  \]
\end{corollary}
While providing some convergence guarantee, note that the probability of
success is low, scaling down linearly with $d$. One way to formally solve
this is to repeat the algorithm $\Omega(d)$ times, which ensures that with
high probability, at least one output will succeed (and finding it can be
done empirically by testing the outputs on a validation set). However, it
turns out that by picking $\bw_0$ in a smarter way, we can get a bound where
the $d$ factors are substantially improved.

Specifically, we consider the following method, parameterized by number of iterations $T_0$, which are implemented before the main algorithm above:
\begin{itemize}
  \item Sample $\bw$ from a standard Gaussian distribution on $\reals^d$
  \item Let $\bw_0=0$.
  \item For $t=1,\ldots,T_0$, let $\bw_0:=\bw_0+\frac{1}{T_0}\tilde{A}_t\bw$
  \item Return $\bw_0:=\frac{\bw_0}{\norm{\bw_0}}$.
\end{itemize}
Essentially, instead of initializing from a random point $\bw$, we initialize from
\[
\frac{\tilde{A}\bw}{\norm{\tilde{A}\bw}}~~,~~ \text{where}~~ \tilde{A}=\frac{1}{T_0}\sum_{t=1}^{T_0}\tilde{A}_t.
\]
Since $\tilde{A}$ is a mean of $T_0$ random matrices with mean $A$, this
amounts to performing a single approximate power iteration. Recently, it was
shown that a single exact power iteration can improve the starting point of
stochastic methods for PCA \cite{shamir2015fast}. The method above extends
this idea to a purely streaming setting, where we only have access to
stochastic approximations of $A$.

The improved properties of $\bw_0$ with this initialization is formalized in the following lemma (where $\norm{A}_F$ denotes the Frobenius norm of $A$):
\begin{lemma}\label{lem:improvestart}
  The following holds for some numerical constants $c,c'>0$: For $\bw_0$ as defined above, if $T_0\geq cdb^2\log(d)$, then with probability at least $\frac{7}{10}-\frac{2}{d}-\exp(-d/8)$,
  \[
  \frac{1}{\inner{\bv,\bw_0}^2}\leq c'\log(d)n_A,
  \]
  where $n_A = \frac{\norm{A}_{F}^2}{\norm{A}^2}$ is the numerical rank of $A$.
\end{lemma}
The proof is provided in \subsecref{subsec:prooflemimprovestart}. 
Combining this with \thmref{thm:main}, we immediately get the following corollary:
\begin{corollary}\label{cor:smart}
If $\bw_0$ is initialized as described above, then \thmref{thm:main} applies with $p=\Ocal(\log(d)n_A)$, and the returned
$\bw$ satisfies, with probability at least $\Omega(1/n_A\log(d))$,
  \[
  1-\frac{\bw^\top A \bw}{\norm{A}} ~\leq~ \Ocal\left(\frac{\log(T)b\sqrt{\log(d)n_A}}{\sqrt{T}}\right),
  \]
\end{corollary}
The improvement of \corref{cor:smart} compared to \corref{cor:vanilla} depends on how much smaller is $n_A$, the numerical rank of $A$, compared to $d$. We argue that in most cases, $n_A$ is much smaller, and often can be thought of as a moderate constant, in which case \corref{cor:smart} provides an $\tilde{\Ocal}\left(\frac{b}{\sqrt{T}}\right)$ error bound with probability $\tilde{\Omega}(1)$, at the cost of $\tilde{\Ocal}(db^2)$ additional iterations at the beginning. Specifically:
\begin{itemize}
  \item $n_A$ is always in $[1,d]$, and in particular, can never be larger than $d$.
  \item $n_A$ is always upper bounded by the rank of $A$, and is small even if $A$ is only approximately low rank. For example, if the spectrum of $A$ has polynomial decay $i^{-\alpha}$ where $\alpha>1$, then $n_A$ will be a constant independent of $d$. Moreover, to begin with, PCA is usually applied in situations where we hope $A$ is close to being low rank.
  \item When $\tilde{A}_t$ is of rank $1$ (which is the case, for instance,
      in PCA, where $\tilde{A}_t$ equals the outer product of the $t$-th
      datapoint $\bx_t$), we have $n_A\leq b^2$, where we recall that $b$
      upper bounds the scaled spectral norm of $\tilde{A}_t$. In machine
      learning application, the data norm is often assumed to be bounded,
      hence $b$ is not too large. To see why this holds, note that for rank
      $1$ matrices, the spectral and Frobenius norms coincide, hence
      \[
      n_A ~=~ \left(\frac{\norm{A}_{F}}{\norm{A}}\right)^2
      ~=~\left(\frac{\norm{\E[\tilde{A}_1]}_{F}}{\norm{A}}\right)^2
      ~\leq~                   \left(\E\left[\frac{\norm{\tilde{A}_1}_{F}}{\norm{A}}\right]\right)^2
      ~=~
      \left(\E\left[\frac{\norm{\tilde{A}_1}}{\norm{A}}\right]\right)^2
      ~\leq~b^2,
      \]
      where we used Jensen's inequality.
\end{itemize}
Similar to \corref{cor:vanilla}, we can also convert the bound of \corref{cor:smart} into a high-probability bound, by repeating the algorithm $\tilde{\Ocal}(n_A)$ times.

\section{Convergence under an Eigengap Assumption}

Although our main interest so far has been the convergence of SGD without any eigengap assumptions, we show in this section that our techniques also imply new bounds for PCA with an eigengap assumptions, which in certain aspects are stronger than what was previously known. 

Specifically, we consider the same setting as before, but where the ratio $\frac{s_1-s_2}{s_1}$, where $s_1,s_2$ are the leading singular values of the covariance matrix $A$ is assumed to be strictly positive and lower bounded by some fixed $\lambda>0$. Using this assumption and a proof largely similar to that of \thmref{thm:main}, we have the following theorem:

\begin{theorem}\label{thm:gap}
	Under the same conditions as \thmref{thm:main}, suppose furthermore that
	\begin{itemize}
		\item The top two eigenvalues of $A$ have a gap $\lambda\norm{A}>0$
		\item 	$\frac{\log^2(T)b^2p}{\lambda T}\leq \frac{\log(T)b\sqrt{p}}{\sqrt{T}}$
	\end{itemize}
	If we run the algorithm above for $T>1$ iterations with
	$\eta = \frac{\log(T)}{\lambda T}$ (assumed to be $\leq 1$), then with probability at least $\frac{1}{c p}$, the returned $\bw$ satisfies
	\[
	1-\frac{\bw^\top A \bw}{\norm{A}} ~\leq~ c'\frac{\log^2(T)b^2 p}{\lambda T},
	\]
	where $c,c'$ are positive numerical constants.
\end{theorem}
The proof appears in \subsecref{subsec:proofgap}. Considering first the technical conditions of the theorem, we note that assuming $\frac{\log^2(T)b^2p}{\lambda T}\leq \frac{\log(T)b\sqrt{p}}{\sqrt{T}}$ simply amounts to saying that $T$ is sufficiently large so that the $\Ocal\left(\frac{\log^2(T)b^2p}{\lambda T}\right)$ bound provided by \thmref{thm:gap} is better than the $\Ocal\left(\frac{\log(T)b\sqrt{p}}{\sqrt{T}}\right)$ bound provided by  \thmref{thm:main}, by more than a constant. This is the interesting regime, since otherwise we might as well choose $\eta$ as in \thmref{thm:main} and get a better bound without any eigengap assumptions. Moreover, as in \thmref{thm:main}, a similar proof would hold if the step size is replaced by $c\log(T)/\lambda T$ for some constant $c\geq 1$.  

As in \thmref{thm:main}, we note that $p$ can be as large as $d$ under random initialization, but this can be improved to the numerical rank of $A$ using an approximate power iteration, or by analyzing the algorithm starting from a warm-start  point $\bw_0$ for which $\frac{1}{\inner{\bv,\bw_0}^2}\leq \Ocal(1)$ for a leading eigenvector $\bv$ of $A$. Also, note that under an eigengap assumption, if $1-\frac{\bw^\top A\bw}{\norm{A}}$ goes to $0$ with the number of iterations $T$, it must hold that $\inner{\bv,\bw}^2$ goes to $1$ for a leading eigenvector of $A$, so the analysis with $p=\Ocal(1)$ is also relevant for analyzing SGD for sufficiently large $T$, once we're sufficiently close to the optimum.

Comparing the bound to previous bounds in the literature for SGD-like methods (which all assume an eigengap, e.g. \cite{balsubramani2013fast,hardt2014noisy,de2015global,JiKaMuNeSi15}), an interesting difference is that the dependence on the eigengap $\lambda$ is only $1/\lambda$, as opposed to $1/\lambda^2$ or worse. Intuitively, we are able to improve this dependence since we track the suboptimality directly, as opposed to tracking how $\bw_T$ converges to a leading eigenvector, say in terms of the Euclidean norm. This has an interesting parallel in the analysis of SGD for $\lambda$-strongly convex functions, where the suboptimality of $\bw_T$ decays as $\tilde{\Ocal}(1/\lambda T)$, although $\E[\norm{\bw_T-\bw^*}^2]$ can only be bounded by $\Ocal(1/\lambda^2T)$ (compare for instance Lemma 1 in \cite{rakhlin2012making} and Theorem 1 in \cite{shamir2013stochastic}). Quite recently, Jin et al. (\cite{JiKaMuNeSi15}) proposed another streaming algorithm which does have only $1/\lambda$ dependence (at least for sufficiently large $T$), and a high probability convergence rate which is even asymptotically optimal in some cases. However, their formal analysis is from a warm-start point (which implies $p=\Ocal(1)$ in our notation), whereas the analysis here applies to any starting point. Moreover, the algorithm in \cite{JiKaMuNeSi15} is different and more complex, whereas our focus here is on the simple and practical SGD algorithm. Finally, we remark that although an $\Ocal(1/\lambda T)$ convergence rate is generally optimal (using any algorithm), we do not know whether the dependence on $b$ and $p$ in the convergence bound of \thmref{thm:gap} for SGD is optimal, or whether it can be improved.

\section{Proofs}

\subsection{Proof of \thmref{thm:main}}\label{subsec:proofmain}

To simplify things, we will assume that we work in a coordinate system where
$A$ is diagonal, $A=\text{diag}(s_1,\ldots,s_d)$, where $s_1\geq
s_2\geq\ldots\geq s_d\geq 0$, and $s_1$ is the eigenvalue corresponding to
$\bv$. This is without loss of generality, since the algorithm and the
theorem conditions are invariant to the choice of coordinate system.
Moreover, since the objective function in the theorem is invariant to
$\norm{A}$, we shall assume that $\norm{A}=s_1=1$. Under these assumptions,
the theorem's conditions reduce to:
\begin{itemize}
	\item $\frac{1}{w_{0,1}^2}\leq p$, for some $p\geq 8$
	\item $b\geq 1$ is an upper bound on
	$\norm{\tilde{A}_t},\norm{\tilde{A}_t-A}$
\end{itemize}

Let $\epsilon\in (0,1)$ be a parameter to be determined later. The proof
works by lower bounding the probability of the objective function (which
under the assumption $\norm{A}=1$, equals $1-\bw^\top A \bw$) being
suboptimal by at most $\epsilon$. This can be written as
\[
\Pr\left(\frac{\bw_T^\top(I-A)\bw_T}{\norm{\bw_T}^2}\leq \epsilon\right),
\]
or equivalently,
\[
\Pr\left(\bw_T^\top((1-\epsilon)I-A)\bw_T\leq 0\right).
\]
Letting
\[
V_T = \bw_T^\top((1-\epsilon)I-A)\bw_T,
\]
we need to lower bound $\Pr(V_T\leq 0)$.

In analyzing the convergence of stochastic gradient descent, a standard
technique to bound such probabilities is via a martingale analysis, showing
that after every iteration, the objective function decreases by a certain
amount. Unfortunately, due to the non-convexity of the objective function
here, the amount of decrease at iteration $t$ critically depends on the
current iterate $\bw_t$, and in the worst case may even be $0$ (e.g. if
$\bw_t$ is orthogonal to the leading eigenvector, and there is no noise).
Moreover, analyzing the evolution of $\bw_t$ is difficult, especially without
eigengap assumptions, where there isn't necessarily some fixed direction
which $\bw_t$ converges to. Hence, we are forced to take a more circuitous
route.

In a nutshell, the proof is composed of three parts. First, we prove that if
$\epsilon$ and the step size $\eta$ are chosen appropriately, then
$\E[V_T]\leq -\tilde{\Omega}\left((1+\eta)^{2T}\frac{\epsilon}{p}\right)$. If
we could also prove a concentration result, namely that $V_T$ is not much
larger than its expectation, this would imply that $\Pr(V_T\leq 0)$ is indeed
large. Unfortunately, we do not know how to prove such concentration.
However, it turns out that it is possible to prove that $V_T$ is not much
\emph{smaller} than its expected value: More precisely, that $V_T \geq
-\tilde{\Ocal}\left((1+\eta)^{2T}\epsilon\right)$ with high probability. We
then show that given such a high-probability \emph{lower bound} on $V_T$, and
a bound on its expectation, we can produce an \emph{upper bound} on $V_T$
which holds with probability $\tilde{\Omega}(1/p)$, hence leading to the
result stated in the theorem.

We begin with a preliminary technical lemma:
\begin{lemma}\label{lem:s}
	For any $\epsilon,\eta\in (0,1)$, and integer $k\geq 0$,
	\[
	\max_{s\in[0,1]}(1+\eta s)^{k}(1-\epsilon-s) ~\leq~ 1+2\frac{\left(1+\eta(1-\epsilon)\right)^k}{\eta(k+1)}.
	\]
\end{lemma}
\begin{proof}
	The result trivially holds for $k=0$, so we will assume $k>0$ from now. Let
	\[
	f(s)=(1+\eta s)^{k}(1-\epsilon-s).
	\]
	Differentiating $f$ and setting to zero, we have
	\begin{align*}
	&k\eta(1+\eta s)^{k-1}(1-\epsilon-s)-(1+\eta s)^k = 0\\
	&\Leftrightarrow~~ k\eta(1-\epsilon-s)= 1+\eta s\\
	&\Leftrightarrow~~ \frac{k\eta(1-\epsilon)-1}{k\eta+\eta}=s\\
	&\Leftrightarrow~~ s=\frac{k(1-\epsilon)-1/\eta}{k+1}.
	\end{align*}
	Let $s_c = \frac{k(1-\epsilon)-1/\eta}{k+1}$ denote this critical point,
	and consider two cases:
	\begin{itemize}
		\item $s_c\notin [0,1]$: In that case, $f$ has no critical points in
		the domain, hence is maximized at one of the domain endpoints, with
		a value of at most
		\[
		\max\{f(0),f(1)\} = \max\{1-\epsilon,-\epsilon(1+\eta)^k\}\leq 1.
		\]
		\item $s_c\in [0,1]$: In that case, we must have
		$k(1-\epsilon)-\frac{1}{\eta}\geq 0$, and the value of $f$ at $s_c$
		is
		\begin{align*}
		&\left(1+\frac{\eta k(1-\epsilon)-1}{k+1}\right)^k\left(1-\epsilon-\frac{k(1-\epsilon)-1/\eta}{k+1}\right)\\
		&= \left(1+\frac{\eta k(1-\epsilon)-1}{k+1}\right)^k\left(\frac{1-\epsilon+\frac{1}{\eta}}{k+1}\right)\\
		&\leq \left(1+\eta(1-\epsilon)\right)^k\left(\frac{1+\frac{1}{\eta}}{k+1}\right)\\
		&\leq
		\frac{2\left(1+\eta(1-\epsilon)\right)^k}{\eta(k+1)}.
		\end{align*}
		The maximal value of $f$ is either the value above, or the maximal
		value of $f$ at the domain endpoints, which we already showed to be
		most $1$. Overall, the maximal value $f$ can attain is at most
		\[
		\max\left\{1,\frac{2\left(1+\eta(1-\epsilon)\right)^k}{\eta(k+1)}\right\}\leq
		1+\frac{2\left(1+\eta(1-\epsilon)\right)^k}{\eta(k+1)}.
		\]
	\end{itemize}
	Combining the two cases, the result follows.
\end{proof}

Using this lemma, we now prove that $V_T = \bw_T^\top ((1-\epsilon)I-A)\bw_T$
has a large negative expected value. To explain the intuition, note that if
we could have used the exact $A$ instead of the stochastic approximations
$\tilde{A}_t$ in deriving $\bw_T$, then we would have
\begin{align*}
\bw_T^\top((1-\epsilon)I-A)\bw_T &= \bw_0^\top (I+\eta A)^{T}((1-\epsilon)I-A)(I+\eta A)^{T}\bw_0\\
&=\sum_{j=1}^{d}(1+\eta s_j)^{2T}(1-\epsilon-s_j)w_{0,j}^2\\
&\leq \frac{1}{p}(1+\eta s_1)^{2T}(1-\epsilon-s_1)+\sum_{j=2}^{d}(1+\eta s_j)^{2T}(1-\epsilon-s_j)w_{0,j}^2\\
&\leq \frac{1}{p}(1+\eta s_1)^{2T}(1-\epsilon-s_1)+\left(\sum_{j=2}^{d}w_{0,j}^2\right)\max_{s\in [0,1]}(1+\eta s)^{2T}(1-\epsilon-s),
\end{align*}
which by the assumptions $s_1=1$ and $1=\norm{\bw_0}^2=\sum_{j=1}^{d}w_{0,j}^2$ is at most
\[
-\frac{\epsilon}{p}(1+\eta)^{2T}+\max_{s\in [0,1]}(1+\eta s)^{2T}(1-\epsilon-s).
\]
Applying \lemref{lem:s} and picking $\eta,\epsilon$ appropriately, it can be
shown that the above is at most
$-\Omega\left(\frac{\epsilon}{p}\left(1+\eta\right)^{2T}\right)$.

Unfortunately, this calculation doesn't apply in practice, since we use the
stochastic approximations $\tilde{A}_t$ instead of $A$. However, using more
involved calculations, we prove in the lemma below that the expectation is
still essentially the same, provided $\epsilon,\eta$ are chosen appropriately.

\begin{lemma}\label{lem:wrw}
	If $\eta = \frac{1}{b}\sqrt{\frac{1}{pT}}\leq 1$ and $\epsilon = c\frac{\log(T)b\sqrt{p}}{\sqrt{T}}\leq 1$ for some sufficiently large constant $c$, then it holds that
	\[
	\E[V_T] ~\leq~ -(1+\eta)^{2T}\frac{\epsilon}{4p}.
	\]
\end{lemma}
\begin{proof}
	To simplify notation, define for all $t=1,\ldots,T$ the matrices
	\[
	C^t_0 = I+\eta A~~~,~~~ C^t_1 = \eta (\tilde{A}_t-A).
	\]
	Note that $C^t_0$ is deterministic whereas $C^t_1$ is random and zero-mean. Moreover, $\norm{C^t_0}\leq 1+\eta$ and $\norm{C^t_1}\leq \eta b$.
	
	By definition of the algorithm, we have the following:
	\begin{align*}
	V_T &=\bw_T^\top ((1-\epsilon)I-A)\bw_T\\
	&=~ \bw_0^\top\left(\prod_{t=1}^{T}\left(I+\eta \tilde{A_t}\right)\right)((1-\epsilon)I-A)\left(\prod_{t=T}^{1}\left(I+\eta\tilde{A}_t\right)\right)\bw_0\\
	&= \bw_0^\top\left(\prod_{t=1}^{T}\left(C^t_0+C^t_1\right)\right)((1-\epsilon)I-A)\left(\prod_{t=T}^{1}\left(C^t_0+C^t_1\right)\right)\bw_0\\
	&=\sum_{(i_1,\ldots,i_T)\in \{0,1\}^T}\sum_{(j_1,\ldots,j_T)\in\{0,1\}^T}\bw_0^\top\left(\prod_{t=1}^{T} C^t_{i_t}\right) ((1-\epsilon)I-A)\left(\prod_{t=T}^{1}C^t_{j_t}\right)\bw_0.
	\end{align*}
	Since $C^1_1,\ldots,C^T_1$ are independent and zero-mean, the expectation of each summand in the expression above is non-zero only if $i_t=j_t$ for all $t$.
	Therefore,
	\[
	\E\left[\bw_T^\top ((1-\epsilon)I-A)\bw_T\right] ~=~ \sum_{(i_1,\ldots,i_T)\in \{0,1\}^T}\E\left[\bw_0^\top\left(\prod_{t=1}^{T} C^t_{i_t}\right) ((1-\epsilon)I-A)\left(\prod_{t=T}^{1}C^t_{i_t}\right)\bw_0\right].
	\]
	We now decompose this sum according to what is the largest value of $t$ for which $i_t=1$ (hence $C^t_{i_t}=C^t_1$). The intuition for this, as will be seen shortly, is that \lemref{lem:s} allows us to attain tighter bounds on the summands when $t$ is much smaller than $T$. Formally, we can rewrite the expression above as
	\begin{align*}
	&\E\left[\bw_0\left(\prod_{t=1}^{T}C^t_0\right)((1-\epsilon)I-A)\left(\prod_{t=T}^{1}C^t_0\right)\bw_0\right]\\
	&~~~~+\sum_{k=0}^{T-1}\sum_{(i_1,\ldots,i_{k})\in\{0,1\}^{k}}\E\left[\bw_0\left(\prod_{t=1}^{k}C^t_{i_t}\right)C^{k+1}_1
	\left(\prod_{t=k+2}^{T}C^t_0\right)((1-\epsilon)I-A)\left(\prod_{t=T}^{k+2}C^t_0\right)C^{k+1}_1\left(\prod_{t=k}^{1}C^t_{i_t}\right)\bw_0\right].
	\end{align*}
	Since $C^t_0=I+\eta A$ is diagonal and the same for all $t$, and $((1-\epsilon)I-A)$ is diagonal as well, we can simplify the above to
	\begin{align*}
	&\bw_0(C^1_0)^{2T}((1-\epsilon)I-A)\bw_0\\
	&~~~~+\sum_{k=0}^{T-1}\sum_{(i_1,\ldots,i_{k})\in\{0,1\}^{k}}\E\left[\bw_0\left(\prod_{t=1}^{k}C^t_{i_t}\right)C^{k+1}_1
	(C^1_0)^{2(T-k-1)}((1-\epsilon)I-A)C^{k+1}_1\left(\prod_{t=k}^{1}C^t_{i_t}\right)\bw_0\right].
	\end{align*}
	Using the fact that the spectral norm is sub-multiplicative, and that for any symmetric matrix $B$, $\bv^\top B\bv\leq \norm{\bv^2}\lambda_{\max}(B)$, where $\lambda_{\max}(B)$ denotes the largest eigenvalue of $B$, we can upper bound the above by
	\begin{align*}
	&\leq~\bw_0(C^1_0)^{2T}((1-\epsilon)I-A)\bw_0\\
	&~~~~~+\sum_{k=0}^{T-1}\sum_{(i_1,\ldots,i_{k})\in\{0,1\}^{k}}\E\left[\norm{\bw_0}^2\left(\prod_{t=1}^{k}\norm{C^t_{i_t}}^2\right)\norm{C^{k+1}_1}^2
	\lambda_{\max}\left((C^1_0)^{2(T-k-1)}((1-\epsilon)I-A)\right)\right].
	\end{align*}
	Since $\norm{\bw_0}=1$, and $\norm{C^t_0}\leq (1+\eta)$, $\norm{C^t_1}\leq \eta b$, this is at most
	\begin{align}
	&\bw_0(C^1_0)^{2T}((1-\epsilon)I-A)\bw_0\notag\\
	&~~~~~+\sum_{k=0}^{T-1}\sum_{(i_1,\ldots,i_{k})\in\{0,1\}^{k}}\left((1+\eta)^{2\left(k-\sum_{t=1}^{k}i_t\right)}(\eta b)^{2\sum_{t=1}^{k}i_t}\right)(\eta b)^2
	\lambda_{\max}\left((C^1_0)^{2(T-k-1)}((1-\epsilon)I-A)\right)\notag\\
	&~= \bw_0(C^1_0)^{2T}((1-\epsilon)I-A)\bw_0\notag\\
	&~~~~~~+\sum_{k=0}^{T-1}\left((1+\eta)^2+(\eta b)^2\right)^k(\eta b)^2
	\lambda_{\max}\left((C^1_0)^{2(T-k-1)}((1-\epsilon)I-A)\right)\notag\\
	&~= \bw_0(I+\eta A)^{2T}((1-\epsilon)I-A)\bw_0\notag\\
	&~~~~~~+(\eta b)^2\sum_{k=0}^{T-1}\left((1+\eta)^2+(\eta b)^2\right)^k
	\lambda_{\max}\left((I+\eta A)^{2(T-k-1)}((1-\epsilon)I-A)\right)\label{eq:bigsum}
	\end{align}
	Recalling that $A=\text{diag}(s_1,\ldots,s_d)$ with $s_1=1$, that
	$\norm{\bw_0}^2=\sum_{j=1}^{d}w_{0,j}^2=1$, and that $w_{0,1}^2\geq \frac{1}{p}$, the first term in \eqref{eq:bigsum} equals
	\begin{align*}
	\bw_0(I+\eta A)^{2T}((1-\epsilon)I-A)\bw_0&= \sum_{j=1}^{d}(1+\eta s_j)^{2T}(1-\epsilon-s_j)w_{0,j}\\
	&= (1+\eta)(-\epsilon)w_{0,1}^2+\sum_{j=2}^{d}(1+\eta s_j)^{2T}(1-\epsilon-s_j)w_{0,j}^2\\
	&\leq -(1+\eta)^{2T}\frac{\epsilon}{p}+\max_{s\in [0,1]}(1+\eta s)^{2T}(1-\epsilon-s).
	\end{align*}
	Applying \lemref{lem:s}, and recalling that $\eta\leq 1$, we can upper bound the above by
	\begin{align}
	&-(1+\eta)^{2T}\frac{\epsilon}{p}+1+2\frac{(1+\eta(1-\epsilon))^{2T}}{\eta(2T+1)}\notag\\
	&=~(1+\eta)^{2T}\left(-\frac{\epsilon}{p}+(1+\eta)^{-2T}+2\frac{\left(\frac{1+\eta(1-\epsilon)}{1+\eta}\right)^{2T}}{\eta(2T+1)}\right)\notag\\
	&\leq~ (1+\eta)^{2T}\left(-\frac{\epsilon}{p}+(1+\eta)^{-2T}+\frac{\left(1-\frac{1}{2}\eta\epsilon\right))^{2T}}{\eta T}\right).
	\label{eq:sum1}
	\end{align}
	As to the second term in \eqref{eq:bigsum}, again using the fact that $A=\text{diag}(s_1,\ldots,s_d)$, we can upper bound it by
	\[
	(\eta b)^2\sum_{k=0}^{T-1}\left((1+\eta)^2+(\eta b)^2\right)^k
	\max_{s\in [0,1]}(1+\eta s)^{2(T-k-1)}(1-\epsilon-s).
	\]
	Applying \lemref{lem:s}, and recalling that $\eta\leq 1$, this is at most
	\begin{align*}
	&(\eta b)^2\sum_{k=0}^{T-1}\left((1+\eta)^2+(\eta b)^2\right)^k
	\left(1+2\frac{(1+\eta(1-\epsilon))^{2(T-k-1)}}{\eta(2(T-k)-1)}\right)\\
	&=~(\eta b)^2(1+\eta)^{2T}\sum_{k=0}^{T-1}\left(1+\left(\frac{\eta b}{1+\eta}\right)^2\right)^k
	\left((1+\eta)^{-2(T-k)}+2\frac{\left(\frac{1+\eta(1-\epsilon)}{1+\eta}\right)^{2(T-k)}}{\eta(2(T-k)-1)}\right)\\
	&\leq~(\eta b)^2(1+\eta)^{2T}\sum_{k=0}^{T-1}\left(1+(\eta b)^2\right)^k
	\left((1+\eta)^{-2(T-k)}+2\frac{\left(1-\frac{1}{2}\eta\epsilon\right)^{2(T-k)}}{\eta(2(T-k)-1)}\right).
	\end{align*}
	Upper bounding $\left(1+(\eta b)^2\right)^k$ by $\left(1+(\eta
	b)^2\right)^{T}$, and rewriting the sum in terms of $k$ instead of $T-k$,
	we get
	\[
	(\eta b)^2(1+\eta)^{2T}\left(1+(\eta b)^2\right)^T\sum_{k=1}^{T}
	\left((1+\eta)^{-2k}+2\frac{\left(1-\frac{1}{2}\eta\epsilon\right)^{2k}}{\eta(2k-1)}\right).
	\]
	Since $k\geq 1$, we have $\frac{1}{2k-1} = \frac{2k}{2k-1}\frac{1}{2k}\leq
	2\frac{1}{2k}$, so the above is at most
	\begin{align*}
	&(\eta b)^2(1+\eta)^{2T}\left(1+(\eta b)^2\right)^T\sum_{k=1}^{T}
	\left((1+\eta)^{-2k}+\frac{4}{\eta}\frac{\left(1-\frac{1}{2}\eta\epsilon\right)^{2k}}{2k}\right)\\
	&\leq~(\eta b)^2(1+\eta)^{2T}\left(1+(\eta b)^2\right)^T
	\left(\sum_{k=1}^{\infty}(1+\eta)^{-2k}+\frac{4}{\eta}\sum_{k=1}^{\infty}\frac{\left(1-\frac{1}{2}\eta\epsilon\right)^{k}}{k}\right)\\
	&=(\eta b)^2(1+\eta)^{2T}\left(1+(\eta b)^2\right)^T\left(\frac{1}{(1+\eta)^2-1}-
	\frac{4}{\eta}\log\left(\frac{1}{2}\eta\epsilon\right)\right)\\
	&\leq(\eta b)^2(1+\eta)^{2T}\left(1+(\eta b)^2\right)^T\left(\frac{1}{2\eta}+
	\frac{4}{\eta}\log\left(\frac{2}{\eta\epsilon}\right)\right)\\
	&=\eta b^2(1+\eta)^{2T}\left(1+(\eta b)^2\right)^T\left(\frac{1}{2}+
	4\log\left(\frac{2}{\eta\epsilon}\right)\right).
	\end{align*}
	Recalling that this is an upper bound on the second term in \eqref{eq:bigsum}, and combining with
	the upper bound in \eqref{eq:sum1} on the first term, we get overall a bound of
	\begin{equation}\label{eq:vff}
	(1+\eta)^{2T}\left(-\frac{\epsilon}{p}+(1+\eta)^{-2T}+\frac{\left(1-\frac{1}{2}\eta\epsilon\right)^{2T}}{\eta T}+\eta b^2\left(1+(\eta b)^2\right)^T\left(\frac{1}{2}+
	4\log\left(\frac{2}{\eta\epsilon}\right)\right)\right).
	\end{equation}
	We now argue that under suitable choices of $\eta,\epsilon$, the expression
	above is $-\Omega((1+\eta)^{2T}(\epsilon/p)$. For example, this is satisfied if $\eta = \frac{1}{b\sqrt{pT}}$,
	and we pick $\epsilon = \frac{c\log(T)b\sqrt{p}}{\sqrt{T}}$ for some sufficiently large constant
	$c$. Under these choices, the expression inside the main parentheses above
	becomes
	\[
	-c\frac{\log(T)b}{\sqrt{pT}}+\left(1+\frac{1}{b\sqrt{pT}}\right)^{-2T}
	+b\sqrt{\frac{p}{T}}\left(1-\frac{c\log(T)}{2T}\right)^{2T}+\frac{b}{\sqrt{pT}}\left(1+\frac{1}{pT}\right)^T\left(\frac{1}{2}+
	4\log\left(\frac{2T}{c\log(T)}\right)\right).
	\]
	Using the facts that $(1-a/t)^t\leq \exp(-a)$ for all positive $t,a$ such that $a/t<1$, and that $c\log(T)/2T<1$ by the assumption that $\epsilon\leq 1$, the above is at most
	\begin{align*}
	&-c\frac{\log(T)b}{\sqrt{pT}}
	+\frac{b}{\sqrt{pT}}\left(p\exp(-c\log(T))+\exp(1/p)\left(\frac{1}{2}+
	4\log\left(\frac{2T}{c\log(T)}\right)\right)\right)+\left(1+\frac{1}{b\sqrt{pT}}\right)^{-2T}\\
	&=~c\frac{\log(T)b}{\sqrt{pT}}\left(-1+\frac{p}{c\log(T)T^c}+\frac{\exp(1/p)}{c\log(T)}
	\left(\frac{1}{2}+
	4\log\left(\frac{2T}{c\log(T)}\right)\right)\right)+\left(1+\frac{1}{b\sqrt{pT}}\right)^{-2T}.
	\end{align*}
	Note that $p,b\geq 1$ by assumption, and that we can assume $T\geq p$ (by the assumption that $\epsilon\leq 1$). Therefore, picking $c$ sufficiently large ensures that the above is at most
	\[
	c\frac{\log(T)b}{\sqrt{pT}}\left(-\frac{1}{2}\right)+\left(1+\frac{1}{b\sqrt{pT}}\right)^{-2T}.
	\]
	The second term is exponentially small in $T$, and in particular can be verified to be less
	than $\frac{1}{4}c\frac{\log(T)b}{\sqrt{pT}}$ in the regime where
	$\epsilon=c\frac{\log(T)b\sqrt{p}}{\sqrt{T}}$ is at most $1$ (assuming $c$ is large enough). Overall, we
	get a bound of $-c\frac{\log(T)b}{\sqrt{pT}}\cdot\frac{1}{4} = -\frac{\epsilon}{4p}$. Plugging
	this back into \eqref{eq:vff}, the result follows.
\end{proof}

Having proved an upper bound on $\E[V_T]$, we now turn to prove a
high-probability lower bound on $V_T$. The proof is based on relating $V_T$ to $\norm{\bw_T}^2$, and then performing a rather straightforward martingale analysis of $\log(\norm{\bw_T}^2)$. 

\begin{lemma}\label{lem:tail}
	Suppose that $\tilde{A}_t$ is positive semidefinite for all $t$, and $\Pr(\norm{\tilde{A}_t}\leq b)=1$. Then for any $\delta\in
	(0,1)$, we have with probability at least $1-\delta$ that
	\[
	V_T > -\exp\left(\eta b\sqrt{T\log(1/\delta)}+(b^2+3)T\eta^2\right)(1+\eta)^{2T}\epsilon.
	\]
\end{lemma}
\begin{proof}
	Since $I-A$ is a positive semidefinite matrix, we have
	\[
	V_T = \bw_T^\top ((1-\epsilon)I-A)\bw_T \geq -\epsilon \norm{\bw_T}^2.
	\]
	Thus, it is sufficient to prove that
	\begin{equation}\label{eq:wtwhatneed}
	\norm{\bw_T}^2 < \exp\left(\eta b\sqrt{T\log(1/\delta)}+(b^2+3)T\eta^2\right)(1+\eta)^{2T}.
	\end{equation}
	
	The proof goes through a martingale argument. We have
	\begin{align*}
	\log(\norm{\bw_T}^2) &= \log\left(\prod_{t=0}^{T-1}\frac{\norm{\bw_{t+1}}^2}{\norm{\bw_t}^2}\right)\\
	&= \sum_{t=0}^{T-1}\log\left(\frac{\norm{\bw_{t+1}}^2}{\norm{\bw_t}^2}\right)\\
	&= \sum_{t=0}^{T-1}\log\left(\frac{\norm{(I+\eta\tilde{A}_t)\bw_t}^2}{\norm{\bw_t}^2}\right)\\
	&= \sum_{t=0}^{T-1}\log\left(1+\left(\frac{\norm{(I+\eta\tilde{A}_t)\bw_t}^2}{\norm{\bw_t}^2}-1\right)\right).
	\end{align*}
	Note that since $\tilde{A}_t$ is positive semidefinite, we always have
	$(1+\eta b)\norm{\bw_t}^2\geq \norm{(I+\eta\tilde{A}_t)\bw_t}^2\geq \norm{\bw_t}^2$, and therefore each
	summand is of the form $\log(1+a_t)$ where $a_t\in [0,\eta b]$.
	Using the identity $\log(1+a)\leq a$ for any non-negative $a$, we can upper bound the above by
	\begin{equation}\label{eq:martsum}
	\sum_{t=0}^{T-1}\left(\frac{\norm{(I+\eta\tilde{A}_t)\bw_t}^2}{\norm{\bw_t}^2}-1\right).
	\end{equation}
	Based on the preceding discussion, this is a sum of random variables
	bounded in $[0,\eta b]$, and the expectation of the $t$-th summand over $\tilde{A}_t$,
	conditioned on $\tilde{A}_1,\ldots,\tilde{A}_{t-1}$, equals
	\begin{align*}
	&\frac{\bw_t^\top \E\left[(I+\eta\tilde{A}_t)^\top(I+\eta\tilde{A}_t)\right]\bw_t}{\norm{\bw_t}^2}-1\\
	&=~\frac{\bw_t^\top \left((I+\eta A)^2+\eta^2\left(\tilde{A}_t^\top\tilde{A}_t-A^2\right)\right)\bw_t}{\norm{\bw_t}^2}-1\\
	&\leq~\frac{\bw_t^\top (I+\eta A)^2\bw_t}{\norm{\bw_t}^2}+\eta^2\frac{\bw_t^\top\tilde{A}_t^\top\tilde{A}_t\bw_t}{\norm{\bw_t}^2}-1\\
	&\leq \norm{(I+\eta A)^2}+\eta^2\norm{\tilde{A}_t^\top\tilde{A}_t}-1\\
	&\leq (1+\eta)^2+\eta^2\norm{\tilde{A}_t}^2-1\\
	&\leq 2\eta+(b^2+1)\eta^2.
	\end{align*}
	Using Azuma's inequality, it follows that with probability at least
	$1-\delta$, \eqref{eq:martsum} is at most
	\[
	T\left(2\eta+(b^2+1)\eta^2\right)+\eta b \sqrt{T\log(1/\delta)}.
	\]
	
	Combining the observations above, and the fact that $\log(1+z)\geq z-z^2$ for any $z\geq 0$, we get that with probability at least $1-\delta$,
	\begin{align*}
	\log(\norm{\bw_T}^2) &<~ 2T\eta+(b^2+1)T\eta^2+\eta b\sqrt{T\log(1/\delta)}\\
	&=~ \eta b\sqrt{T\log(1/\delta)}+(b^2+3)T\eta^2+2T(\eta-\eta^2)\\
	&\leq~ \eta b\sqrt{T\log(1/\delta)}+(b^2+3)T\eta^2+2T\log(1+\eta),
	\end{align*}
	and therefore
	\[
	\norm{\bw_T}^2 < \exp\left(\eta b\sqrt{T\log(1/\delta)}+(b^2+3)T\eta^2\right)(1+\eta)^{2T},
	\]
	which establishes \eqref{eq:wtwhatneed} and proves the lemma.
\end{proof}

We now have most of the required components to prove \thmref{thm:main}.
First, we showed in \lemref{lem:wrw} that if
$\eta=\frac{1}{b}\sqrt{\frac{1}{pT}}$, then
\begin{equation}\label{eq:vvv1}
\E[V_T]\leq -(1+\eta)^{2T}\frac{\epsilon}{4p}.
\end{equation}
for $\epsilon = \Ocal(b\log(T)\sqrt{p/T})$. Using the same step size $\eta$,
\lemref{lem:tail} implies that
\[
\Pr\left(V_T \leq -\exp\left(\sqrt{\frac{\log(1/\delta)}{p}}+\frac{1+3/b^2}{p}\right)(1+\eta)^{2T}\epsilon\right) \leq \delta,
\]
and since we assume $b\geq 1$ (hence $1+3/b^2\leq 4$), this implies that
\begin{equation}\label{eq:vvv2}
\Pr\left(-\frac{V_T}{\exp(4/p)(1+\eta)^{2T}\epsilon} \geq \exp\left(\sqrt{\frac{\log(1/\delta)}{p}}\right)\right) \leq \delta.
\end{equation}
Now, define the non-negative random variable
\[
R_T = \max\left\{0,-\frac{V_T}{\exp(4/p)(1+\eta)^{2T}\epsilon}\right\},
\]
and note that by its definition, $\E[R_T] \geq \E\left[-\frac{V_T}{\exp(4/p)(1+\eta)^{2T}\epsilon}\right]$ and $\Pr\left(R_T\geq \exp\left(\sqrt{\frac{\log(1/\delta)}{p}}\right)\right)$ equals $\Pr\left(-\frac{V_T}{\exp(4/p)(1+\eta)^{2T}\epsilon}\geq\exp\left(\sqrt{\frac{\log(1/\delta)}{p}}\right)\right)$. Using 
\eqref{eq:vvv1} and \eqref{eq:vvv2}, this implies that
\[
\E[R_T]\geq \frac{1}{4p\exp(4/p)}~~~,~~~\Pr\left(R_T \geq \exp\left(\sqrt{\frac{\log(1/\delta)}{p}}\right)\right) \leq \delta.
\]
To summarize the development so far, we defined a non-negative random
variable $R_T$, which is bounded with high probability, yet its expectation
is at least $\Omega(1/p)$. The following lemma shows that for a bounded
non-negative random variable with ``large'' expectation, the probability of it being on the same order as its expectation cannot be too small:
\begin{lemma}\label{lem:inverse}
	Let $X$ be a non-negative random variable such that for some $\alpha,\beta\in [0,1]$, we have $\E[X]\geq \alpha$, and for any $\delta\in (0,1]$,
	\[
	\Pr\left(X\geq \exp\left(\beta\sqrt{\log(1/\delta)}\right)\right)\leq \delta.
	\]
	Then
	\[
	\Pr\left(X>\frac{\alpha}{2}\right) ~\geq~ \frac{\alpha-\exp\left(-\frac{2}{\beta^2}\right)}{15}.
	\]
\end{lemma}
Before proving the lemma, let us show to use it to prove \thmref{thm:main}.
Applying it on the random variable $R_T$, which satisfies the lemma
conditions with $\alpha = \frac{1}{4p\exp(4/p)}$, $\beta=\sqrt{\frac{1}{p}}$,
we have
\begin{align*}
\frac{1}{15}\left(\frac{1}{4p\exp(4/p)}-\exp\left(-2p\right)\right)~&\leq~\Pr\left(R_T > \frac{1}{8p\exp(4/p)}\right) \\
&=~\Pr\left(\max\left\{0,-\frac{V_T}{\exp(4/p)(1+\eta)^{2T}\epsilon}\right\} > \frac{1}{8p\exp(4/p)}\right)\\
&=~\Pr\left(-\frac{V_T}{\exp(4/p)(1+\eta)^{2T}\epsilon} > \frac{1}{8p\exp(4/p)}\right)\\
&=~\Pr\left(V_T\leq -\frac{(1+\eta)^{2T}\epsilon}{8p}\right)\\
&\leq~\Pr\left(V_T\leq 0\right)\\
\end{align*}
$\frac{1}{15}\left(\frac{1}{4p\exp(4/p)}-\exp\left(-2p\right)\right)$
can be verified to be at least $\frac{1}{100 p}$ for any $p\geq 8$, hence we
obtained
\[
\Pr(V_T\leq 0)\geq \frac{1}{100 p}.
\]
As discussed at the beginning of the proof, $V_T\leq 0$ implies that
\[
\frac{\bw_T(I-A)\bw_T}{\norm{\bw_T}^2}\leq \epsilon,
\]
where $\epsilon = c\frac{\log(T)b\sqrt{p}}{\sqrt{T}}$ is the value chosen in
\lemref{lem:wrw}, and the theorem is established.

All that remains now is to prove \lemref{lem:inverse}. To explain the intuition, suppose that $X$ in the lemma was actually at most $1$ with probability $1$, rather than just bounded with high probability. Then we would have
\begin{align*}
\alpha&\leq \E[X]~=~\Pr\left(X\geq\frac{\alpha}{2}\right)\E\left[X|X\geq\frac{\alpha}{2}\right]+\Pr\left(X<\frac{\alpha}{2}\right)\E\left[X|X\leq \frac{\alpha}{2}\right]\\
&\leq \Pr\left(X\geq \frac{\alpha}{2}\right)\cdot 1+\Pr\left(X<\frac{\alpha}{2}\right)\cdot \frac{\alpha}{2}\\
&= \Pr\left(X\geq\frac{\alpha}{2}\right)+\left(1-\Pr\left(X\geq \frac{\alpha}{2}\right)\right)\frac{\alpha}{2},
\end{align*}
which implies that 
\[\alpha~\leq~ \left(1-\frac{\alpha}{2}\right)\Pr\left(X\geq\frac{\alpha}{2}\right)+\frac{\alpha}{2}~~\Longrightarrow~~ \Pr\left(X\geq \frac{\alpha}{2}\right)~\geq~ \frac{\alpha/2}{1-\alpha/2}~\geq~\frac{\alpha}{2}.
\]
Therefore, $X$ is at least one-half its expectation lower bound ($\alpha$) with probability at least $\alpha/2$. The proof of \lemref{lem:inverse}, presented below, follows the same intuition, but uses a more delicate analysis since $X$ is actually only upper bounded with high probability.

\begin{proof}[Proof of \lemref{lem:inverse}]
	Inverting the bound in the lemma, we have that for any $z\in [1,\infty)$,
	\[
	\Pr(X\geq z)\leq \exp(-(\log(z)/\beta)^2).
	\]
	Now, let $r_2>r_1>0$, be parameters to be chosen later. We have
	\begin{align}
	\E[X] = \int_{z=0}^{\infty}\Pr(X>z)dz &= \int_{z=0}^{r_1}\Pr(X>z)dz+\int_{z=r_1}^{r_2}\Pr(X>z)dz+\int_{z=r_2}^{\infty}\Pr(X>z)dz\notag\\
	&\leq r_1+(r_2-r_1)\Pr(X>r_1)+\int_{z=r_2}^{\infty}\exp(-(\log(z)/\beta)^2)dz\label{eq:rrr}
	\end{align}
	Performing the variable change $y=(\log(z)/\beta)^2$ (which implies
	$z=\exp(\beta\sqrt{y})$ and
	$dy=\frac{2\sqrt{y}}{\exp(\beta\sqrt{y})}dz$), we get
	\begin{align*}
	\int_{z=r_2}^{\infty}\exp(-(\log(z)/\beta)^2)dz &= \int_{y=\left(\frac{\log(r_2)}{\beta}\right)^2}^{\infty}\frac{1}{2\sqrt{y}}\exp(\beta\sqrt{y}-y)dy\\
	&\leq \frac{\beta}{2\log(r_2)}\int_{y=\left(\frac{\log(r_2)}{\beta}\right)^2}^{\infty}\exp(\beta\sqrt{y}-y)dy.
	\end{align*}
	Suppose that we choose $r_2\geq \exp(2\beta^2)$. Then $\frac{\log(r_2)}{2\beta}\geq
	\beta$, which implies that for any $y$ in the integral above, $\frac{1}{2}\sqrt{y}\geq
	\beta$, and therefore $\beta\sqrt{y}-y\leq \frac{1}{2}y-y=-\frac{1}{2}y$. As a result, we can upper bound the above by
	\[
	\frac{\beta}{2\log(r_2)}\int_{y=\left(\frac{\log(r_2)}{\beta}\right)^2}^{\infty}\exp\left(-\frac{1}{2}y\right)dy
	~=~ \frac{\beta}{\log(r_2)}\exp\left(-\frac{\log^2(r_2)}{2\beta^2}\right).
	\]
	Plugging this upper bound back into \eqref{eq:rrr}, extracting
	$\Pr(X>r_1)$, and using the assumption $\E[X]\geq \alpha$, we get that
	\[
	\Pr(X>r_1) ~\geq~ \frac{\alpha-r_1-\frac{\beta}{\log(r_2)}\exp\left(-\frac{\log^2(r_2)}{2\beta^2}\right)}{r_2-r_1}.
	\]
	Choosing $r_1=\alpha/2$ and $r_2=\exp(2)$ (which ensures $r_2\geq
	\exp(2\beta^2)$ as assumed earlier, since $\beta\leq 1$), we get
	\[
	\Pr\left(X>\frac{\alpha}{2}\right) ~\geq~ \frac{\alpha-\beta\exp\left(-\frac{2}{\beta^2}\right)}{2\exp(2)-\alpha}.
	\]
	Since $\beta,\alpha\leq 1$, and $2\exp(2)< 15$, this can be simplified to
	\[
	\Pr\left(X>\frac{\alpha}{2}\right) ~\geq~ \frac{\alpha-\exp\left(-\frac{2}{\beta^2}\right)}{15}.
	\]
\end{proof}

\subsection{Proof of \lemref{lem:improvestart}}\label{subsec:prooflemimprovestart}

  Define $\Delta = \norm{\tilde{A}-A}$. Also, let $s_1\geq s_2\geq\ldots\geq s_d\geq 0$ be the $d$ eigenvalues of $A$, with eigenvectors $\bv_1,\ldots,\bv_d$, where we assume that $\bv=\bv_1$. Using the facts $(x+y)^2\leq 2x^2+2y^2$ and $\norm{\bv_1}=1$, we have
  \begin{align*}
  \frac{1}{\inner{\bv_1,\bw_0}^2} &= \frac{\norm{\tilde{A}\bw}^2}{\inner{\bv_1,\tilde{A}\bw}^2} ~=~
  \frac{\norm{A\bw+(\tilde{A}-A)\bw}^2}{\left(\inner{\bv_1,A\bw}+\inner{\bv_1,(\tilde{A}-A)\bw}\right)^2}\\
  &\leq~ \frac{2\norm{A\bw}^2+2\norm{(\tilde{A}-A)\bw}^2}{\inner{\bv_1,A\bw}^2+2\inner{\bv_1,A\bw}\inner{\bv_1,(\tilde{A}-A)\bw}}
  ~\leq~
  \frac{2\norm{A\bw}^2+2\norm{\bw}^2\Delta^2}{\inner{\bv_1,A\bw}^2-2|\inner{\bv_1,A\bw}|\norm{\bw}\Delta},
  \end{align*}
  where we implicitly assume that $\Delta$ is sufficiently small for the denominator to be positive (eventually, we will pick $T_0$ large enough to ensure this).
  
  Recall that $\bv_1,\ldots,\bv_d$ forms an orthonormal basis for $\reals^d$, so $\bw=\sum_{i=1}^{d}\bv_i\inner{\bv_i,\bw}$. Therefore, we can write the above as
  \begin{align*}
  \frac{2\left(\sum_{i=1}^{d}s_i\bv_i\inner{\bv_i,\bw}\right)^2+2\norm{\bw}^2\Delta^2}{\left(s_1\inner{\bv_1,\bw}\right)^2-2|s_1\inner{\bv_1,\bw}|\norm{\bw}\Delta}
  &~=~ \frac{2\sum_{i=1}^{d}s_i^2\inner{\bv_i,\bw}^2+2\norm{\bw}^2\Delta^2}{s_1^2 \inner{\bv_1,\bw}^2-2|s_1\inner{\bv_1,\bw}|\norm{\bw}\Delta}\\
  &~\leq~
  \frac{2\left(\sum_{i=1}^{d}s_i^2\right)\left(\max_i\inner{\bv_i,\bw}^2\right)+2\norm{\bw}^2\Delta^2}{s_1^2 \inner{\bv_1,\bw}^2-2|s_1\inner{\bv_1,\bw}|\norm{\bw}\Delta}.
  \end{align*}
  To simplify notation, since $\bw$ is drawn from a standard Gaussian distribution, which is rotationally invariant, we can assume without loss of generality that $(\bv_1,\ldots,\bv_d)=(\be_1,\ldots,\be_d)$, the standard basis, so the above reduces to
  \[
  \frac{2\left(\sum_{i=1}^{d}s_i^2\right)\max_i w_i^2+2\norm{\bw}^2\Delta^2}{s_1^2 w_1^2-2|s_1w_1|\norm{\bw}\Delta}.
  \]
  Recall that $\Delta=\norm{\tilde{A}-A}$, where $\tilde{A}$ is the average of $T_0$ independent random matrices with mean $A$, and spectral norm at most $\norm{A}b$.
  Using a Hoeffding matrix bound (e.g. \cite{tropp2012user}), and the fact that $\norm{A}=s_1$, it follows that with probability at least $1-\delta$,
  \[
  \Delta\leq \norm{A}b\sqrt{\frac{8\log(d/\delta)}{T_0}} ~=~ s_1 \sqrt{\frac{8b^2\log(d/\delta)}{T_0}}.
  \]
  Plugging into the above, we get an upper bound of
  \[
  \frac{2\left(\sum_{i=1}^{d}s_i^2\right)\max_i w_i^2+\norm{\bw}^2s_1^2 \frac{16b^2\log(d/\delta)}{T_0}}{s_1^2 w_1^2-2s_1^2|w_1|\norm{\bw} \sqrt{\frac{8b^2\log(d/\delta)}{T_0}}},
  \]
  holding with probability at least $1-\delta$. Dividing both numerator and denominator by
  $s_1^2$, and recalling that $n_A = \frac{\norm{A}_F^2}{\norm{A}^2} = \frac{\sum_{i=1}^{d}s_i^2}{s_1^2}$, the above equals
  \begin{equation}\label{eq:ndnd}
  \frac{2n_A \max_i w_i^2+\norm{\bw}^2\frac{16b^2\log(d/\delta)}{T_0}}{w_1^2-2|w_1|\norm{\bw} \sqrt{\frac{8b^2\log(d/\delta)}{T_0}}}
  ~=~
  \frac{2n_A \max_i w_i^2+\norm{\bw}^2\frac{16b^2\log(d/\delta)}{T_0}}{|w_1|\left(|w_1| -2\norm{\bw}\sqrt{\frac{8b^2\log(d/\delta)}{T_0}}\right)}.
  \end{equation}
  Based on standard Gaussian concentration arguments, it holds that
  \[
  \Pr\left(w_1^2\leq \frac{1}{8}\right)\leq \frac{3}{10}~~,~~
  \Pr\left(\max_i w_i^2\geq 18\log(d)\right)\leq \frac{1}{d}~~,~~
  \Pr\left(\norm{\bw}\geq \sqrt{2d}\right) \leq \exp\left(-\frac{d}{8}\right).
  \]
  (see for instance the proof of Lemma 1 in \cite{shamir2015fast}, and Corollary 2.3 in \cite{barvinok2005}). Combining the above with a union bound, it holds that with probability at least $1-\delta-\frac{3}{10}-\frac{1}{d}-\exp(-d/8)$, \eqref{eq:ndnd} is at most \[
  \frac{36\log(d)n_A+\frac{32db^2\log(d/\delta)}{T_0}}{\frac{1}{8}\left(\frac{1}{8} -2\sqrt{2}\sqrt{\frac{8db^2\log(d/\delta)}{T_0}}\right)}.
  \]
  Recalling that this is an upper bound on $\frac{1}{\inner{\bv_1,\bw_0}^2}$, picking $\delta=1/d$ for simplicity, and slightly simplifying, we showed that with probability at least $\frac{7}{10}-\frac{2}{d}-\exp(-d/8)$,
  \[
  \frac{1}{\inner{\bv_1,\bw_0}^2}~\leq~ \frac{36\log(d)n_A+\frac{64 b^2\log(d)}{T_0}}{\frac{1}{8}\left(\frac{1}{8} -8\sqrt{2}\sqrt{\frac{db^2\log(d)}{T_0}}\right)}.
  \]
  Since $n_A\geq 1$, then by picking $T_0\geq cdb^2\log(d)$ for a sufficiently large constant $c$, we get that $\frac{1}{\inner{\bv_1,\bw_0}^2}\leq c'\log(d)n_A$ for a numerical constant $c'$, as required.

\subsection{Proof of \thmref{thm:gap}}\label{subsec:proofgap}

The proof is very similar to that of \thmref{thm:main}, using some of the same lemmas, and other lemmas having slight differences to take advantage of the eigengap assumption. Below, we focus on the differences, referring to parts of the proof of \thmref{thm:main} where necessary.

First, as in the proof of \thmref{thm:main}, we assume that we work in a coordinate system where
$A$ is diagonal, $A=\text{diag}(s_1,\ldots,s_d)$, where $s_1\geq
s_2\geq\ldots\geq s_d\geq 0$, and $s_1$ is the eigenvalue corresponding to
$\bv$. By the eigengap assumption, we can assume that $s_2,\ldots,s_d$ are all at most $1-\lambda$ for some strictly positive $\lambda\in (0,1]$. Under these assumptions,
the theorem's conditions reduce to:
\begin{itemize}
	\item $\frac{1}{w_{0,1}^2}\leq p$, for some $p\geq 8$
	\item $b\geq 1$ is an upper bound on
	$\norm{\tilde{A}_t},\norm{\tilde{A}_t-A}$,
\end{itemize}
and as in the proof of \thmref{thm:main}, it is enough to lower bound $\Pr(V_T\leq 0)$ where
\[
V_T = \bw_T^\top((1-\epsilon)I-A)\bw_T.
\]

We begin by a technical lemma, which bounds a certain quantity appearing later in the proofs:
\begin{lemma}\label{lem:regime_gap}
	Under the conditions of \thmref{thm:gap},
	\[
	\frac{\log^2(T)b^2}{\lambda^2 T}\leq \frac{1}{p} \leq 1.
	\]
\end{lemma}
\begin{proof}
	By the assumption $\frac{\log^2(T)b^2p}{\lambda T}\leq \frac{\log(T)b\sqrt{p}}{\sqrt{T}}$, it follows that $\frac{\log(T)b}{\lambda\sqrt{T}}\leq \frac{1}{\sqrt{p}}$, and the result follows by squaring both sides.
\end{proof}

We now continue by presenting the following variant of \lemref{lem:wrw}:

\begin{lemma}\label{lem:wrw_gap}
	Under the conditions of \thmref{thm:gap}, if we pick $\eta = \frac{\log(T)}{\lambda T}\leq 1$ and $\epsilon = c\frac{\log^2(T)b^2p}{\lambda T}$ for some sufficiently large numerical constant $c$, then 
	\[
	\E[V_T] ~\leq~ -(1+\eta)^{2T}\frac{\epsilon}{4p}.
	\]
\end{lemma}
\begin{proof}
	By the exact same proof as in \lemref{lem:wrw} (up till \eqref{eq:bigsum}), we have
	\begin{align}
		\E[V_T] &=\E[\bw_T^\top ((1-\epsilon)I-A)\bw_T]\notag\\
     		&~\leq \bw_0(I+\eta A)^{2T}((1-\epsilon)I-A)\bw_0\notag\\
     		&~~~~~~+(\eta b)^2\sum_{k=0}^{T-1}\left((1+\eta)^2+(\eta b)^2\right)^k
     		\lambda_{\max}\left((I+\eta A)^{2(T-k-1)}((1-\epsilon)I-A)\right)\label{eq:bigsum_gap}
	\end{align}
	Recalling that $A=\text{diag}(s_1,\ldots,s_d)$ with $s_1=1$, that
	$\norm{\bw_0}^2=\sum_{j=1}^{d}w_{0,j}^2=1$, and that $w_{0,1}^2\geq \frac{1}{p}$, the first term in \eqref{eq:bigsum} equals
	\begin{align}
		\bw_0(I+\eta A)^{2T}((1-\epsilon)I-A)&\bw_0~= \sum_{j=1}^{d}(1+\eta s_j)^{2T}(1-\epsilon-s_j)w_{0,j}^2\notag\\
		&= (1+\eta)(-\epsilon)w_{0,1}^2+\sum_{j=2}^{d}(1+\eta s_j)^{2T}(1-\epsilon-s_j)w_{0,j}^2\notag\\
		&\leq -(1+\eta)^{2T}\frac{\epsilon}{p}+\max_{s\in [0,1-\lambda]}(1+\eta s)^{2T}(1-\epsilon-s)\notag\\
		&\leq -(1+\eta)^{2T}\frac{\epsilon}{p}+(1+\eta(1-\lambda))^{2T}\notag\\&\leq (1+\eta)^{2T}\left(-\frac{\epsilon}{p}+\left(1-\frac{\eta\lambda}{1+\eta}\right)^{2T}\right)\notag\\
		&\leq
		(1+\eta)^{2T}\left(-\frac{\epsilon}{p}+\left(1-\frac{\eta\lambda}{2}\right)^{2T}\right),\label{eq:sum1_gap}
	\end{align}
	where we used the assumption that $\eta\leq 1$. 
	As to the second term in \eqref{eq:bigsum_gap}, upper bounding it in exactly the same way as in the proof of \lemref{lem:wrw} (without using the eigengap assumption), we get an upper bound of 
	\[
	\eta b^2(1+\eta)^{2T}\left(1+(\eta b)^2\right)^T\left(\frac{1}{2}+
		4\log\left(\frac{2}{\eta\epsilon}\right)\right).
	\]
	Combining this with \eqref{eq:sum1_gap}, and plugging back to \eqref{eq:bigsum_gap}, we get that
	\begin{equation}\label{eq:vff_gap}
		\E[V_T]~\leq~
		(1+\eta)^{2T}\left(-\frac{\epsilon}{p}+\left(1-\frac{\eta\lambda}{2}\right)^{2T}+\eta b^2\left(1+(\eta b)^2\right)^T\left(\frac{1}{2}+
		4\log\left(\frac{2}{\eta\epsilon}\right)\right)\right).
	\end{equation}
	Picking $\eta=\frac{\log(T)}{\lambda T}$, and $\epsilon=\frac{c\log^2(T)b^2 p}{\lambda T}$ for some constant $c\geq 2$, the above equals
	\[
	(1+\eta)^{2T}\left(-\frac{c\log^2(T)b^2}{\lambda T}+\left(1-\frac{\log(T)}{2T}\right)^{2T}+\frac{b^2\log(T)}{\lambda T}\left(1+\frac{b^2\log^2(T)}{\lambda^2 T^2}\right)^T\left(\frac{1}{2}+
	4\log\left(\frac{2\lambda^2 T^2}{c\log^3(T)b^2 p}\right)\right)\right).
	\]
	Using the facts that $(1+a/t)^t\leq \exp(a)$ for all positive $t,a$, that $c\log^3(T)b^2 p\geq 2$, and that $\lambda\leq 1$, the above is at most
	\[
	(1+\eta)^{2T}\left(-\frac{c\log^2(T)b^2}{\lambda T}+\frac{1}{T}+\frac{b^2\log(T)}{\lambda T}\exp\left(\frac{b^2\log^2(T)}{\lambda^2 T}\right)\left(\frac{1}{2}+
	4\log\left(T^2\right)\right)\right).
	\]
	By \lemref{lem:regime_gap}, $\frac{b^2\log^2(T)}{\lambda^2 T}\leq 1$, so the above is at most
	\begin{align*}
	&(1+\eta)^{2T}\left(-\frac{c\log^2(T)b^2}{\lambda T}+\frac{1}{T}+\frac{b^2\log(T)}{\lambda T}\exp(1)\left(\frac{1}{2}+
	8\log\left(T\right)\right)\right)\\	
	&\leq (1+\eta)^{2T}\frac{b^2\log^2(T)}{\lambda T}\left(-c+\frac{\lambda}{b^2\log^2(T)}+\exp(1)\left(\frac{1}{2\log(T)}+8\right)\right).
	\end{align*}
	Clearly, for large enough $c$, the expression in the main parenthesis above is at most $-c/4$, so we get an upper bound of
	\[
	-(1+\eta)^{2T}\frac{cb^2\log^2(T)}{4\lambda T} ~=~
	-(1+\eta)^{2T}\frac{\epsilon}{4p},
	\]
	from which the result follows.
\end{proof}

Rather similar to the proof of \thmref{thm:main}, we now define the non-negative random variable 
\[
R_T = \max\left\{0,-\frac{V_T}{\exp((b^2+3)T\eta^2)(1+\eta)^{2T}\epsilon}\right\}.
\]
By \lemref{lem:wrw_gap}, 
\[
\E[R_T]~\geq~ \E\left[-\frac{V_T}{\exp((b^2+3)T\eta^2)(1+\eta)^{2T}\epsilon}\right]~\geq~ \frac{1}{4p\exp((b^2+3)T\eta^2)},
\]
and by \lemref{lem:tail},
\[
\Pr\left(R_T \geq \exp\left(\eta b\sqrt{T\log(1/\delta)}\right)\right)\leq \delta.
\]
Therefore, applying \lemref{lem:inverse} on $R_T$, with $\alpha=\frac{1}{4p\exp((b^2+3)T\eta^2)}$ (which is in $[0,1]$) and with $\beta=\eta b\sqrt{T}$ (which can be verified to be in $[0,1]$ by the fact that $\eta=\frac{\log(T)}{\lambda T}$ and \lemref{lem:regime_gap}), we get that
\begin{equation}
\Pr\left(R_T>\frac{1}{8p\exp((b^2+3)T\eta^2)}\right) ~\geq~ \frac{1}{15}\left(\frac{1}{4p\exp((b^2+3)T\eta^2)}-\exp\left(-\frac{2}{\eta^2 b^2 T}\right)\right)\label{eq:gapfinal}.
\end{equation}
By definition of $R_T$, the left hand side of this inequality is at most
\begin{align*}
	&=~\Pr\left(\max\left\{0,-\frac{V_T}{\exp((b^2+3)T\eta^2)(1+\eta)^{2T}\epsilon}\right\} > \frac{1}{8p\exp((b^2+3)T\eta^2)}\right)\\
	&=~\Pr\left(-\frac{V_T}{\exp((b^2+3)T\eta^2)(1+\eta)^{2T}\epsilon} > \frac{1}{8p\exp((b^2+3)T\eta^2)}\right)\\
	&=~\Pr\left(V_T\leq -\frac{(1+\eta)^{2T}\epsilon}{8p}\right)\\
	&\leq~\Pr\left(V_T\leq 0\right),
\end{align*}
and the right hand side of \eqref{eq:gapfinal} (by definition of $\eta$, the assumption $b\geq 1$, and \lemref{lem:regime_gap}) equals
\begin{align*}
&\frac{1}{15}\left(\frac{1}{4p\exp\left(\frac{(b^2+3)\log^2(T)}{\lambda^2 T}\right)}-\exp\left(-\frac{2\lambda^2T}{ b^2 \log^2(T)}\right)\right)\\
&\geq~\frac{1}{15}\left(\frac{1}{4p\exp\left(\frac{4b^2\log^2(T)}{\lambda^2 T}\right)}-\frac{1}{\exp\left(2\frac{\lambda^2T}{b^2\log^2(T)}\right)}\right)\\
&\geq~\frac{1}{15}\left(\frac{1}{4p\exp\left(\frac{4}{p}\right)}-\frac{1}{\exp\left(2p)\right)}\right),
\end{align*}
which can be verified to be at least $\frac{1}{100p}$ for any $p\geq 8$. Plugging these bounds back to \eqref{eq:gapfinal}, we obtained
\[
\Pr(V_T\leq 0)\geq \frac{1}{100 p}.
\]
By definition of $V_T$, $V_T\leq 0$ implies that
\[
\frac{\bw_T(I-A)\bw_T}{\norm{\bw_T}^2}\leq \epsilon,
\]
where $\epsilon = c\frac{\log^2(T)b^2p}{\lambda T}$ is the value chosen in
\lemref{lem:wrw_gap}, and the theorem is established.

\subsubsection*{Acknowledgments}
This research is supported in part by an FP7 Marie Curie CIG grant, the Intel
ICRI-CI Institute, and Israel Science Foundation grant 425/13. We thank Ofer Zeitouni for several illuminating discussions.

\bibliographystyle{plain}
\bibliography{mybib}

\end{document}